\documentclass[sigconf,authorversion]{acmart}

\AtBeginDocument{%
  \providecommand\BibTeX{{%
    \normalfont B\kern-0.5em{\scshape i\kern-0.25em b}\kern-0.8em\TeX}}}

\setcopyright{acmcopyright}
\copyrightyear{2023}
\acmYear{2023}
\acmDOI{XXXXXXX.XXXXXXX}

\acmConference[KDD ’23]{the 29th ACM SIGKDD Conference on Knowledge Discovery and Data Mining}{Aug 06--Aug 10,
  2023}{Long Beach, CA}
\acmPrice{15.00}
\acmISBN{978-1-4503-XXXX-X/18/06}

\usepackage{newtxmath}
\usepackage[utf8]{inputenc} 
\usepackage[T1]{fontenc}    
\usepackage{hyperref}       
\usepackage{url}            
\usepackage{booktabs}       
\usepackage{amsfonts}       
\usepackage{nicefrac}       
\usepackage{microtype}      
\usepackage{xcolor}         
\usepackage{algorithm2e}
\usepackage{relsize}
\usepackage{amsmath}
\usepackage{epsfig}
\usepackage{multirow}
\usepackage{wrapfig,lipsum,booktabs}
\usepackage{flafter}
\usepackage[section]{placeins}
\usepackage{tikz}
\usepackage{balance}
\usetikzlibrary{shapes.arrows}
\newtheorem{theorem}{Theorem}
\newtheorem{assumption}{Assumption}
\newtheorem{definition}{Definition}
\newtheorem{lemma}{Lemma}

\newcommand\etc{etc\@ifnextchar.{}{.\@}}

\definecolor{ao(english)}{rgb}{0.0, 0.5, 0.0}
\newcommand{\gr}[1]{{\textcolor{ao(english)}{#1}}}
\newcommand{\bfgr}[1]{{\bf \textcolor{ao(english)}{#1}}}

\newcommand{\FancyDownArrow}{\begin{tikzpicture}[baseline=-0.3em]
\node[single arrow,draw,rotate=270,single arrow head extend=0.15em,inner
ysep=0.15em,transform shape,line width=0.03em,top color=green,bottom color=green!50!black] (X){};
\end{tikzpicture}}

\newcommand{\FancyUpArrow}{\begin{tikzpicture}[baseline=-0.3em]
\node[single arrow,draw,rotate=90,single arrow head extend=0.15em,inner
ysep=0.15em,transform shape,line width=0.03em,top color=green,bottom color=green!50!black] (X){};
\end{tikzpicture}}

\newcommand{\NA}{---}

\copyrightyear{2023}
\acmYear{2023}
\setcopyright{acmlicensed}\acmConference[KDD '23]{Proceedings of the 29th ACM SIGKDD Conference on Knowledge Discovery and Data Mining}{August 6--10, 2023}{Long Beach, CA, USA}
\acmBooktitle{Proceedings of the 29th ACM SIGKDD Conference on Knowledge Discovery and Data Mining (KDD '23), August 6--10, 2023, Long Beach, CA, USA}
\acmPrice{15.00}
\acmDOI{10.1145/3580305.3599875}
\acmISBN{979-8-4007-0103-0/23/08}

\begin{document}

\title{Neural Insights for Digital Marketing Content Design}

\author{Fanjie Kong}
\authornote{Work done while at Amazon. \newline}
\affiliation{%
  \institution{Duke University, USA}
  \country{}
}
\email{fanjie.kong@duke.edu}

\author{Yuan Li}
\affiliation{%
  \institution{Amazon.com, Inc., USA}
  \country{}
}

\email{liayuan@amazon.com}

\author{Houssam Nassif}
\authornotemark[1]
\affiliation{%
  \institution{Meta, USA}
  \country{}
}
\email{houssamn@meta.com}

\author{Tanner Fiez}
\affiliation{%
  \institution{Amazon.com, Inc., USA}
  \country{}
}
\email{fieztann@amazon.com}

\author{Ricardo Henao}
\affiliation{%
  \institution{Duke University, USA\\
  KAUST, KSA}
  \country{}
}
\email{ricardo.henao@duke.edu	}
\author{Shreya Chakrabarti}
\affiliation{%
  \institution{Amazon.com, Inc., USA}
  \country{}
}
\email{chashrey@amazon.com}
\renewcommand{\shortauthors}{Fanjie Kong et al.}

\newcommand{\etal}{\textit{et al.}}
\newcommand{\ie}{\textit{i.e.}}
\begin{abstract}
In digital marketing, experimenting with new website content is one of the key levers to improve customer engagement. 
However, creating successful marketing content is a manual and time-consuming process that lacks clear guiding principles.
This paper seeks to close the loop between content creation and online experimentation by offering marketers AI-driven actionable insights based on historical data to improve their creative process. 
We present a neural-network-based system that scores and extracts insights from a marketing content design. Namely, a multimodal neural network predicts the attractiveness of marketing contents, and a {\em post-hoc} attribution method generates actionable insights for marketers to improve their content in specific marketing locations. 
Our insights not only point out the advantages and drawbacks of a given current content, but also provide design recommendations based on historical data.
We show that our scoring model and insights work well both quantitatively and qualitatively.
\vspace{2mm}
\end{abstract}

\ccsdesc[500]{Human-centered computing~Interactive systems and tools}
\ccsdesc[500]{Information systems~Retrieval models and ranking}
\ccsdesc[500]{Information systems~Users and interactive retrieval}
\ccsdesc[500]{Information systems~Information retrieval query processing}
\ccsdesc[300]{Computing methodologies~Natural language processing}
\ccsdesc[300]{Computing methodologies~Computer vision}

\keywords{Digital marketing, interactive system, deep learning, model interpretation, image and text recommendation. \newline }


\maketitle

\RestyleAlgo{ruled}
\SetKwComment{Comment}{/* }{ */}



\begin{figure} [!t]
\includegraphics[width=0.45\textwidth]{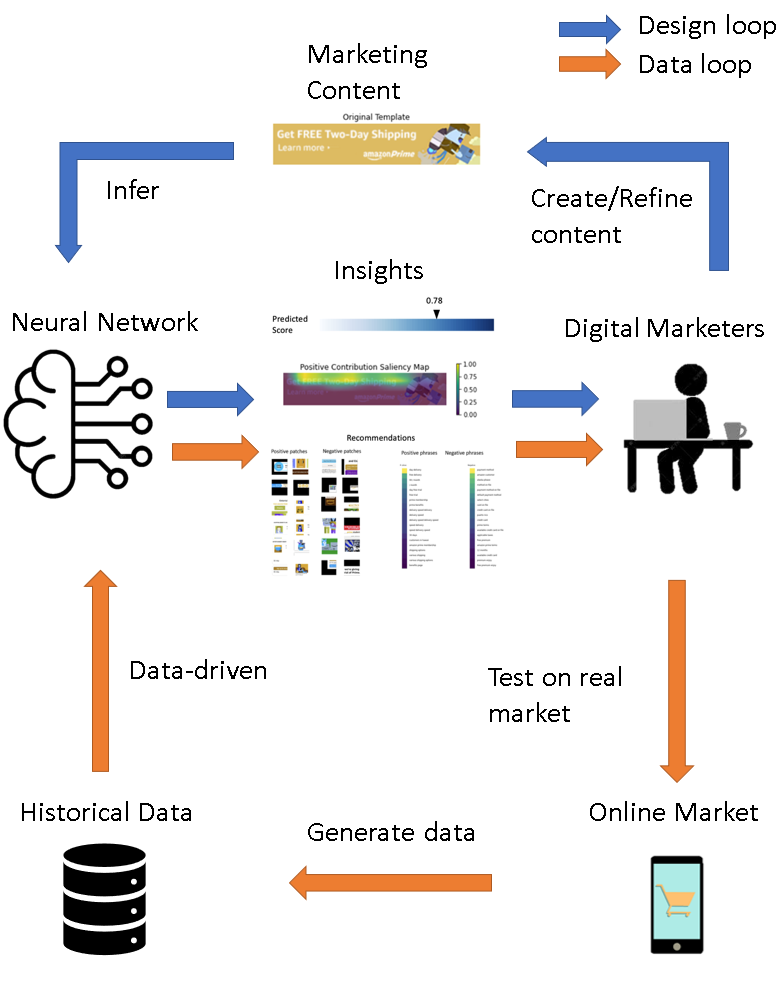}
\caption{Diagram of AI-driven marketing content design. }\label{fig:diagram}
\end{figure}

\section{Introduction}


Content experimentation plays an important role in driving key performance indicators as part of present-day online marketing~\cite{kohavi2017online,Fiez2022Anduril}.
In a typical industrial workflow, digital marketers manually design content, launch controlled online experiments, and receive feedback through collected impression logs.
While this process has proven to be reliable for measuring the incremental impact of content creation, it fails to provide insights to the marketer that can improve the likelihood of future experiments being successful. 
Indeed, unless treatments are deliberately designed relative to a control, it is difficult to establish the source of causality in an experiment outcome.
This limits the opportunity to learn the preferences of a customer base. 
Similarly, the outcomes of online experiments do not immediately provide information to a marketer on how they should design novel content for future experiments.

As a result of the existing content experimentation paradigm, creating new marketing elements is a manual and time-consuming process with significant human involvement.
Successful experiments are often the result of
subject matter expertise among marketing teams, manual detection of patterns across campaigns, and sequential testing of ideas~\cite{Nabi2022EB, zhaoqi2022InstanceOptimal}.
Consequently, it is common that resulting insights suffer from cognitive and incentive bias by marketing teams who analyze the results~\cite{insightBias}.

An opportunity exists to significantly improve the efficiency and effectiveness of marketing content design through data-driven actionable insights.
A fundamental challenge to this objective is that extracting actionable content creation insights from data-driven models requires methods that are interpretable by a human.
Consequently, existing work in this direction has relied on simple machine learning techniques to model digital marketing content.
In the closest related work on the topic~\cite{sinha2020designing}, a generalized linear model with handcrafted features was developed to score marketing content and provide insights.
Despite the promise of this approach, the technique suffers from several shortcomings in real-world marketing scenarios including: $i$) high prediction error, $ii$) a limited number of features, $iii$) the inability to generalize to content with novel features, and $iv$) unclear actionability from interpretation results.

In this paper, we develop a neural-network-based system that scores and extracts insights from a marketing content design to close the loop between content creation and online experimentation (see Figure~\ref{fig:diagram}). 
This approach is motivated by the remarkable success of deep neural nets in diverse application areas~\cite{kong2022efficient, grigorescu2020survey, singh2017machine, wang2022few, geng2020PQR, Geng2023SAmQ, cai2017real,grislain2019recurrent, nuruzzaman2018survey, wang2021personalized, biswas2019seeker}. 
However, providing insights to improve content design is challenging and different from the traditional tasks where deep learning has proven to be successful.
This is due to the fact that predictive performance is not the only objective, but it is also necessary to interpret the model, which is challenging given that deep learning models are generally difficult to interpret black-boxes. 

We overcome this issue by using {\em post-hoc} model agnostic attribution methods.
In summary, our paper makes the following contributions:
\begin{enumerate}
    \item To the best of our knowledge, we may be the first to apply deep learning in the digital marketing design process. We provide an analysis of how to leverage neural network interpretations to help in digital marketing design, and propose a novel image and text insight-generation framework based on attributions from deep neural nets.
    \item We present interpretable insights in an interactive visual format, with actionable insights overlaid with the content.
\end{enumerate}
We validate the performance of the scoring model on an Amazon industry dataset.
We also benchmark a variety of interpretation methods using a novel evaluation scheme.
To the best of our knowledge, this is the first work to apply deep learning as a tool to model digital marketing content and provide insights to improve content design. Lastly, we publicly release the pseudo-code of algorithms described in this paper for
researchers to easily reproduce the code and run our pipeline on their own datasets. Besides, to facilitate replications in other industrial settings, we do share images of our interactive dashboard in Figure~\ref{fig:dashboard}. 

\paragraph{\textbf{Organization.}}
In Section~\ref{sec: data}, we introduce the workflow that describes how digital marketers conduct experiments. In Section~\ref{sec: model}, we present the neural network model used to model the content data and the process used to train the neural network.
In Section~\ref{sec: insights}, we explain the method proposed to generate insights for content based on our multimodal neural network.
In Section~\ref{sec: insights_eval}, we propose a three-step approach to quantitatively evaluate the performance of our insights with respect to the correlation between applying insights-guided modification and the observed outcome.
Finally, in Section~\ref{sec: results}, we discuss our experiments and their results.

\section{Dataset and Metric}\label{sec: data}
Controlled experiments, also called randomized experiments or A/B tests, have had a profound influence in multiple fields, including
medicine, agriculture, manufacturing, and advertising \cite{kohavi2017online,Fiez2022Anduril}. 
Randomized and properly designed experiments can be used to establish causality, that is, to identify elements in marketing content likely to provide incremental impact~\cite{Sawant18HVAE}.
In this paper, our goal is to use neural networks to model digital marketing experiments, and learn causal effects from interpreting the behavior of the model.

A typical marketing dataset consists of multiple sequences of controlled experiments conducted by marketers in different digital marketing locations.
The dataset used in this paper contains tens of thousands of distinct content items and corresponding success rates.
Each marketing content includes various modalities, for instance, an image $I$ corresponding to the web-page screenshot of the content, a text $T$ that contains all textual campaigns in the content, a string $D$ that indicates the marketing content domain and location, and a set of categorical features $F$ that are extracted from the raw content with (potentially) handcrafted functions.


The target metric we adopt is the success rate.
In a binary setting, success can be defined as a click, a purchase, or other valuable customer action.
Using clicks as an example, success rate is the number of clicks over the number of times the content is shown:
\begin{equation}  
  Y = N_{\text{clicks}}/N_{\text{total}},
\end{equation}
where $N_{\text{total}}$ is the total number of people who viewed the content and $N_{\text{clicks}}$ is the number of people who clicked on it.
Our goal is to predict the success rate $Y$ using the multimodal input $X$, while providing insights by interpreting the model and its predictions.

\section{Marketing Content Neural Model}\label{sec: model}
We now introduce the details and components of our marketing content scoring model.
As a working example, we represent a marketing content using four of its modalities: image $I$, text $T$, content domain $D$, and feature vector $F$.
We encode each modality using a corresponding widely-used and efficient neural architecture (see Equation~\ref{eq: score_1}).
The image encoder is an RGB ResNet-18~\cite{he2016deep} model without the fully-connected classifier.
The text encoder is a standard BERT model~\cite{devlin2018bert} without the classification head.
Fully-connected MLP neural networks~\cite{rumelhart1985learning} serve as the encoders for both domain and categorical features.
The details of these networks are in Appendix~\ref{sec: architectures}.
We then use the most basic fusion strategy~\cite{gadzicki2020early} by concatenating the embeddings from all modalities via their encoders (see Equation~\ref{eq: score_2}).
Finally, we feed the concatenated embeddings into another fully-connected MLP neural network for regression. 

Formally, given input content $X=\{I, T, D, F\}$, the corresponding embedding is given by $X_{\text{emb}}=\{I_{\text{emb}}, T_{\text{emb}}, D_{\text{emb}}, F_{\text{emb}}\}$, where
\begin{equation} \label{eq: score_1}
 \begin{array}{l}
   I_{\text{emb}} = \text{ResNet}(I), \quad
   T_{\text{emb}} = \text{BERT}(T), \\
   D_{\text{emb}} = \text{MLP}_1(D), \quad
   F_{\text{emb}} = \text{MLP}_2(F).
\end{array}
\end{equation}
Then, denoting $C(\cdot)$ as the final module which takes all modalities as input, the success rate prediction $\widehat{y}$ is given as follows:
\begin{equation}\label{eq: score_2}
 \begin{array}{l}
   \widehat{y} = C(X_{\text{emb}}) = \text{MLP}_3(\{I_{\text{emb}}, T_{\text{emb}}, D_{\text{emb}}, F_{\text{emb}}\}).
\end{array}
\end{equation}
%
To facilitate model convergence, each sub-network in the multi-modal model is pretrained separately.
We begin by appending a classification head after each encoder to allow it to predict the success rate.
Then, we train each module using a view of the dataset that only contains the respective modality.
Importantly, the regression network $C(\cdot)$ is not trained since we do not have access to its input (concatenated embeddings of all modalities) at this (pretraining) stage.
After pretraining each sub-network, the whole multi-modal network is trained on the multi-modality dataset.
The encoders are initialized with the weights obtained in the pretraining stage.
We report the single-modality sub-networks and final model performance metrics in Section~\ref{sec: results}.

Since we want to predict the continuous, but bounded, success rate $Y$, we append a sigmoid function $\sigma(\cdot)$ after the output $\widehat{y}$ of the final regression function $C(\cdot)$.
Our optimization objective is the mean-squared error (MSE) between $Y$ and $\sigma(\widehat{y})$:
\begin{equation}\label{eq: objective}
 \begin{array}{l}
    L = \text{MSE}(Y, \sigma(\widehat{y})).
\end{array}
\end{equation} 

\section{Neural Insights}\label{sec: insights}
In this section, we describe how we utilize {\em post-hoc} interpretation methods to produce insights from our scoring model.
A key advantage of {\em post-hoc} interpretation is that it can be constructed from an arbitrary prediction model.
This property alleviates the need to rely on customized model architectures for interpretable predictions~\cite{fukui2019attention, wang2019sharpen} or to train separate modules to explicitly produce model explanations~\cite{chang2018explaining, goyal2019counterfactual}.
This section begins by motivating the utility of insights {\em post-hoc} attribution, then describes the attribution methods, and concludes by explaining how we develop insights from attribution techniques. Note that we are formulating a new problem in deep learning, where our insights aim to help marketers improve existing content.

\subsection{Insights: guidance to improve current design}
We start by addressing the attribution problem~\cite{sundararajan2017axiomatic, baehrens2010explain}, defined as the assignment of contributions to individual input features~\cite{efron2020prediction}.
The aim of this subsection is illustrative; we seek to show in a near-ideal scenario that {\em post-hoc} attributions from a neural network can help improve the success rate of content that is being developed, whereas Section~\ref{sec:evaluation} verifies it empirically.
Toward this goal, let us define the input content as a bag a features with a success rate.
\begin{definition}
The input content $X$ is a bag of features $X = \{x_i \in \mathbb{R}^n | i=1,2, ..., N\}$ with common success rate label $Y \in [0, 1]$.
\end{definition}
We now assume that the underlying success rate $Y$ corresponding to content $X$ can be represented as a linear combination of attribution scores for each feature in the representation of $X$.
\begin{assumption}
Given a tuple $\{X,Y\}$, let $\{y_i\in \mathbb{R} | i=1,2,...,N\}$ be the contribution of features of $X$ to the ground-truth success rate $Y$, such that $\sum y_i = Y$.
Each individual attribution $y_i$ corresponds to an individual input feature $x_i$.
We only have access to the bag label $Y$, while the ground-truth feature-level attribution $y_i$ is unknown.
\label{assump:linear}
\end{assumption}
We define an {\em attributor} as a function that estimates the contribution $y_i$ of a feature $x_i\in X$ to the success rate prediction for the entire bag $X$. For example, a digital marketer has a set of promotional slogans $\{x_1, ..., x_r\}$, the contribution of each slogan to the success rate is $\{y_1, ..., y_r\}$. After adding these slogans to a blank content, the success rate of the blank content increased by an increment of $Y = \sum_{i=1,...,r} y_i$. Our attributor predicts the contribution of each slogan in the content such that $c(x_i) = y_i \forall i=1,..., r$.
\begin{definition}
Given a prediction function $C(\cdot)$ such that $C(X)$ predicts $Y$, define an attributor $c(\cdot)$ as a function that estimates the contributions of each input feature $x_i\in X$ to the prediction $C(\cdot)$, which can be expressed as $C(X) =  \sum_{x \in X}c(x)$. 
\label{def:attributor}
\end{definition}
We now use this framework to show that using a feature with a higher attribution score than an existing feature would increase the overall success rate in a near-ideal scenario.
This underscores that attribution methods can act as a guide for digital marketers to refine their existing content given that this effect can also be validated empirically (see Section~\ref{sec:evaluation}).

Consider replacing a feature $x$ in bag $X$ with another feature $\bar{x}'$ such that $c(\bar{x}') \geq c(\bar{x})$, which is consistent with an A/B testing in which a treatment is derived from a control~\cite{bilovs2016open,Fiez2022Anduril}. 
Let $X'$ be the treated content $X' = (X \backslash \{\bar{x}\} )\cup \{\bar{x}'\}$, where $\bar{x}$ is replaced by $\bar{x}'$. 
We now show that the treated content $X'$ will have a higher success rate under certain assumptions. 
\begin{proposition}
Replacing a feature $\bar{x}$ in bag $X$ with a feature $\bar{x}'$ such that $c(\bar{x}')\geq c(\bar{x})$ will increase the overall success rate from $Y$ to $Y'$ when $C(X')\geq C(X)\Leftrightarrow Y' \geq Y$, and under Assumption~\ref{assump:linear}.  
\end{proposition}
\begin{proof}
By Definition~\ref{def:attributor}, $C(X) = \sum_{x \in X} c(x)$.
Thus, since $c(\bar{x}')\geq c(\bar{x})$ by construction, we have that 
\begin{equation}
C(X') =  \sum_{x \in X } c(x) + (c(\bar{x}') - c(\bar{x})) \geq C(X).
\end{equation}
Since $C(X')\geq C(X)\Rightarrow Y' \geq Y$, we conclude that $Y' \geq Y$.
\end{proof}
The above example indicates that replacing features with higher $c(x)$ would increase $Y$ when $C(X)$ is positively correlated with $Y$. In real-world datasets, the condition $C(X') \geq C(X) \Rightarrow Y' \geq Y$ may not always hold. However, we use pairwise accuracy to evaluate the accuracy of our predictor when comparing two content elements in Section~\ref{sec:evaluation} and validate the efficacy of the replacement.
Below, we detail both the prediction function $C(\cdot)$ and attributor $c(\cdot)$.

\subsection{Post-hoc attribution methods}
There are three common trends in mechanisms behind {\em post-hoc} attribution.
Back-propagation-based methods compute attributions according to the gradients with respect to the input~\cite{sundararajan2017axiomatic}. 
Activation-based methods use a variety of ways to weigh activation maps of intermediate layers in neural network to assign attributions~\cite{selvaraju2017grad}.
Perturbation-based methods treat the network as a black-box and assign importance by observing the change in the output after perturbing the inputs.
For instance, feature ablation~\cite{merrick2019randomized} is done by replacing each input feature with a given baseline (zero vector), and computing the difference in the output.
Another alternative is by approximating Shapley values in deep neural networks~\cite{lundberg2017unified, ancona2019explaining}.
Kernel SHAP leverages a kernel-weighted linear regression to estimate the Shapley values of each input as the attribution scores~\cite{lundberg2017unified}.
\citet{langlois2021passive} use PCA to aggregate a variety of attribution methods to estimate the {\em shared} component of the variance between different types of attention maps.

In our implementation, we borrow directly from the mentioned {\em post-hoc} attribution methods, namely, GradCam, Integrated Gradient, Kernel SHAP, Feature Ablation, and PCA, to approximate the attributor $c(\cdot)$.
If the prediction function is a multimodal neural network $C(X)$ as defined in Section~\ref{sec: model}, the attributor is given by $c(x_i) = \text{attribution}(C(X))[i]$.
Note that attributions are rescaled to satisfy $C(X) = \sum_{i=1}^n c(x_i)$, as in Definition~\ref{def:attributor}.


\begin{figure*}[!t]
\includegraphics[width=.9\textwidth]{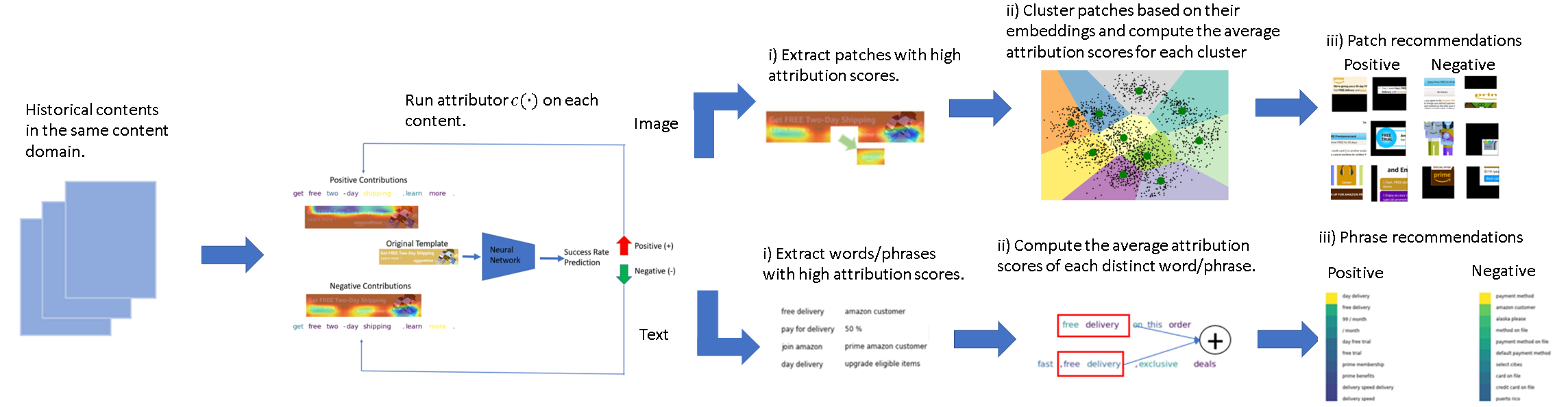}
\caption{Generating image and text recommendations. For all content instances in the same marketing location $D$, we first run attribution methods on ResNet-18 and BERT separately to get the attribution maps for text and images. For text data, the scores are simply ranked by the average of the overlaid attribution values on the same words or phrases. For image data, we crop the salient area in every image, and then cluster them based on their embeddings in ResNet-18. Their scores are ranked by the average attribution scores in the same cluster. See Section~\ref{sec: insights} for details.
}
\label{fig:insights_recommendations}
\end{figure*}

\subsection{Insights: recommending design elements}
%
Given this attribution framework, our system leverages historical data to provide recommendations of visual and textual design element alterations (see Figure~\ref{fig:insights_recommendations}).
The goal of our recommendation is to identify features that are highly likely to improve the success rate of a content being iterated on, and to provide hints for marketers as they embark on designing brand new creative content.
We recommend features ranked by their mean attribution score across the whole dataset.
We compute the rank score $r$ of a feature $x$ as: 
\begin{equation}
  r(x) = \frac{1}{N} \sum_{X \in \mathcal{X}} c(x) ,
\end{equation}
where the rank score $r(x)$ is an estimate of the expected attribution score over the data distribution $\mathcal{X}$.

We split the implementation of our recommendation strategy according to its modality, whether text or image.
We explain our recommendation strategy below and illustrate it in Figure~\ref{fig:insights_recommendations}.

\begin{minipage}[l]{0.95\columnwidth}
\centering
        \includegraphics[width=1.0\textwidth]{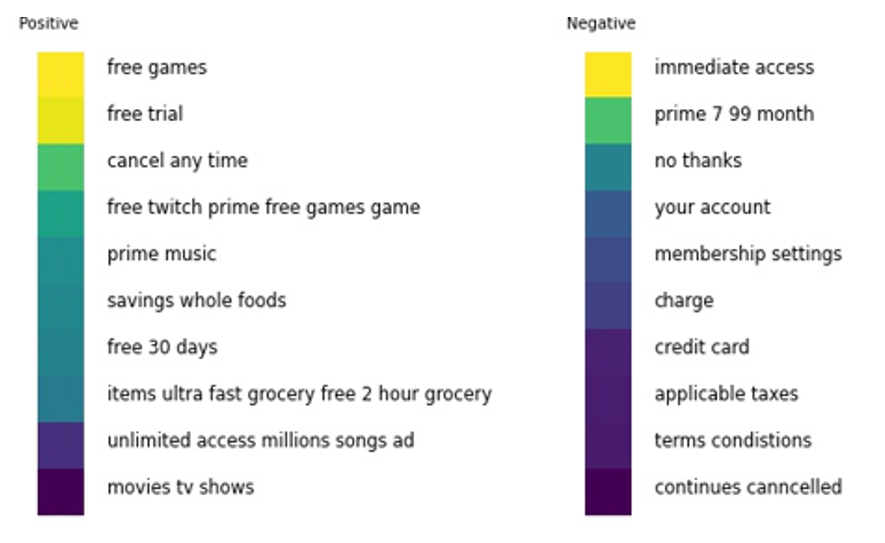}
        \captionof{figure}{Example of marketing recommended phrases.}
        \label{fig: phrase_recommendations}
 \end{minipage}
 
\paragraph{\textbf{Text Recommendation}}
For text data, we include word-level and phrase-level recommendations.
In word-level recommendations, we simply recommend words that have high average attribution scores across all text contents in a marketing location.
In phrase-level recommendations, $i$) we use phrasemachine~\cite{handler2016bag} to extract phrases from each single text content; $ii$) we then compute the attribution score of a phrase by averaging the attribution scores of all its words; $iii$) finally, we recommend phrases that have high average attribution scores across all text contents within a domain.
We define {\em positive} phrases as phrases with the top-10 rank scores while {\em negative} phrases are phrases with the bottom-10 rank scores.

Figure~\ref{fig: phrase_recommendations} shows an illustrative example of top and bottom scoring phrases.
In the positive phrases, our model recommends using slogans about benefits such as ``free game'', ``free trial'', ``free twitch'', ``unlimited access millions songs'', {\em etc}.
The negative phrases are about pricing, payment and legal terms, such as ``prime 7. 99  month'', ``credit card'', ``applicable taxes'', {\em etc}.
\paragraph{\textbf{Image Recommendation}}
For image data, we overlay historical ground truth attributions on top of the image in consideration, recommending actions on patches (subsets) of the image.
While recent works show the success of deep neural networks in image recommendations~\cite{he2016vista,sulthana2020improvising, niu2018neural,biswas2019seeker}, our image recommendation zooms into the salient patches inside images, aiming to provide users with key visual elements that contribute most to the label of the image.
The goal of our image recommendation algorithm is very different from that of existing works.
Moreover, our methodology to use attributions to find salient patches and cluster them to detect common patterns is innovative.

\begin{figure*}
\includegraphics[width=1.0\textwidth]{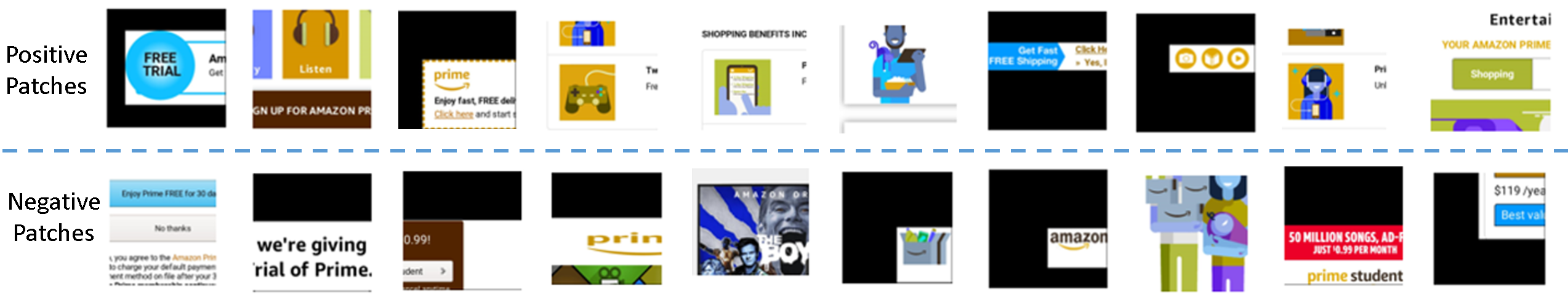}
\caption{An example of visual insights. We show 10 positive patches within the top-10 rank scores and 10 negative patches within bottom-10 rank scores. }
\label{fig: patches_recommendations}
\end{figure*}

Our image recommendation consists of the following steps.
$i$) We first run the attribution method on each single content in the whole dataset.
Then, we extract patches with top-$K$ attribution scores in an image.
Once a patch is selected, we execute non-maximum suppression on the region of the selected patch to ensure each patch is distinct.
$ii$) Subsequently, we cluster these patches based on their ResNet-18 embeddings using K-Means clustering~\cite{ball1965isodata} to uncover the design patterns of these patches.
$iii$) Finally, patch recommendations are collected from each cluster.

In our work, we leverage K-means clustering to help us group similar image patches, as it has been successfully used for unsupervised image classification~\cite{ranjan2017hyperspectral, anas2017skin}.
We use the elbow method to select the number $K$ of centroids~\cite{thorndike1953belongs}.
In order to encourage a diverse set of suggestions, we randomly sample an equal number of patches from each cluster, as different clusters reflect distinct visual information.
This procedure ensures the image recommendations have enough variety of patterns and avoids recommending repetitive patches. 
The positive patches are randomly sampled from clusters within the top-10 rank scores while the negative patches are randomly sampled from clusters within the 10 lowest rank scores.

Figure~\ref{fig: patches_recommendations} shows an illustrative example of our visual design recommendation.
The recommendations of images have some insights similar to text insights in Figure~\ref{fig: phrase_recommendations}.
Some positive patches are illustrations about benefits and some negative patches are illustrations about payment (row 2, column 10) and offers without revealing discounts and upgrades (row 2, column 2).
This example seems to suggest that using the icon of prime (row 1, column 3) is more attractive than the generic Amazon icon (row 2, column 7).
Moreover, negative patches shows a distorted Prime logo (row 2, column 4), an exaggerated human face (row 2, column 5) and an infantile cartoon (row 2, column 8), characters that resonate less with many customers~\cite{Resonantads}, while positive patches recommend entertainment icons (row 1 in columns 2, 4, 8) and more favorable human illustrations such as upbeat smiling persons (row 1, columns 6, 9).

\section{Insights Evaluation}\label{sec: insights_eval}
We now tackle the open-ended problem of evaluating insights. 
We need a practical insight-evaluation metric that marketers can track and trust, that captures the relationship between acting on an insight and its ensuing causal effect, and conveys the expected success rate increase if that insight is applied.
Existing evaluation metrics of interpretation methods span faithfulness, stability and fairness~\cite{agarwal2022openxai}, which do not satisfy our needs. 
\citet{runge2019detecting} quantify the strength of causal relationships from observational time series data with pairwise correlations. 
In our work, we aim to examine the relationship between insights-guided modifications and the ensuing change in the actual success rate.
However, evaluating our insights is an inherently difficult problem since no explicit ground truth feature-level attributions $y$ exist. 

\begin{algorithm}
\caption{A generic three-step approach to evaluate insights of attributor $c(\cdot)$.}\label{alg:eval_general}
\KwData{
Input pairs of control bags and treatment bags $(X, X'), \, \forall X, X' \in \mathcal{X}$ and $(Y, Y'),\,  \forall Y, Y' \in [0, 1]$ are the pairs of control labels and treatment labels respectively, 
and the evaluated attributor $c(\cdot): \mathbb{R}^n \rightarrow [0, 1]\subset \mathbb{R}$.
}
\KwResult{Correlation coefficient $\rho$.
} 
      \par  \textbf{Step $i$).} Compute the distinct elements set $S$, such that the attributes in $S$ can be only found in $X$ or $X'$. \\
   \qquad $S:=\{x|(x \in X \wedge x\notin X')\cup (x \notin X \wedge x\in X') \}$; 
      \par \textbf{Step $ii$).} Compute predicted attribution difference $d_C$ and actual success rate improvement  $d_Y$: \\ 
        \qquad $\mathbf{d_C} := \text{sign}(Y'-Y)(\sum\limits_{x \in (X' \cap S)}c(x) - \sum\limits_{x \in (X \cap S)}c(x))$; \\
       \qquad $\mathbf{d_Y} := |\Delta Y|$; \\
\textbf{Step $iii$).} Examine the linear relationship of variable $\mathbf{d_C}$ and variable $\mathbf{d_Y}$ by computing the Pearson Correlation $\rho$ on the whole dataset. \\
\textbf{Output $\rho$}.
\end{algorithm}

In Section~\ref{sec: insights}, we show that one can leverage insights to improve content attractiveness with an optimal prediction function $C(\cdot)$.
However, in real-world situations, we may not be able to obtain a model $C(\cdot)$ satisfying the idealized properties.
Further, a content change could span multiple features of the original design. 
Similar to~\cite{zhang2021absolute}, which computes the correlation of absolute neighbour differences to detect heteroscedastic relationships, we use the Pearson correlation between the predicted attribution difference and the actual success rate improvement to quantify the relative performance of an an insight.

\begin{figure}
    \centering
    \includegraphics[width=0.45\textwidth]{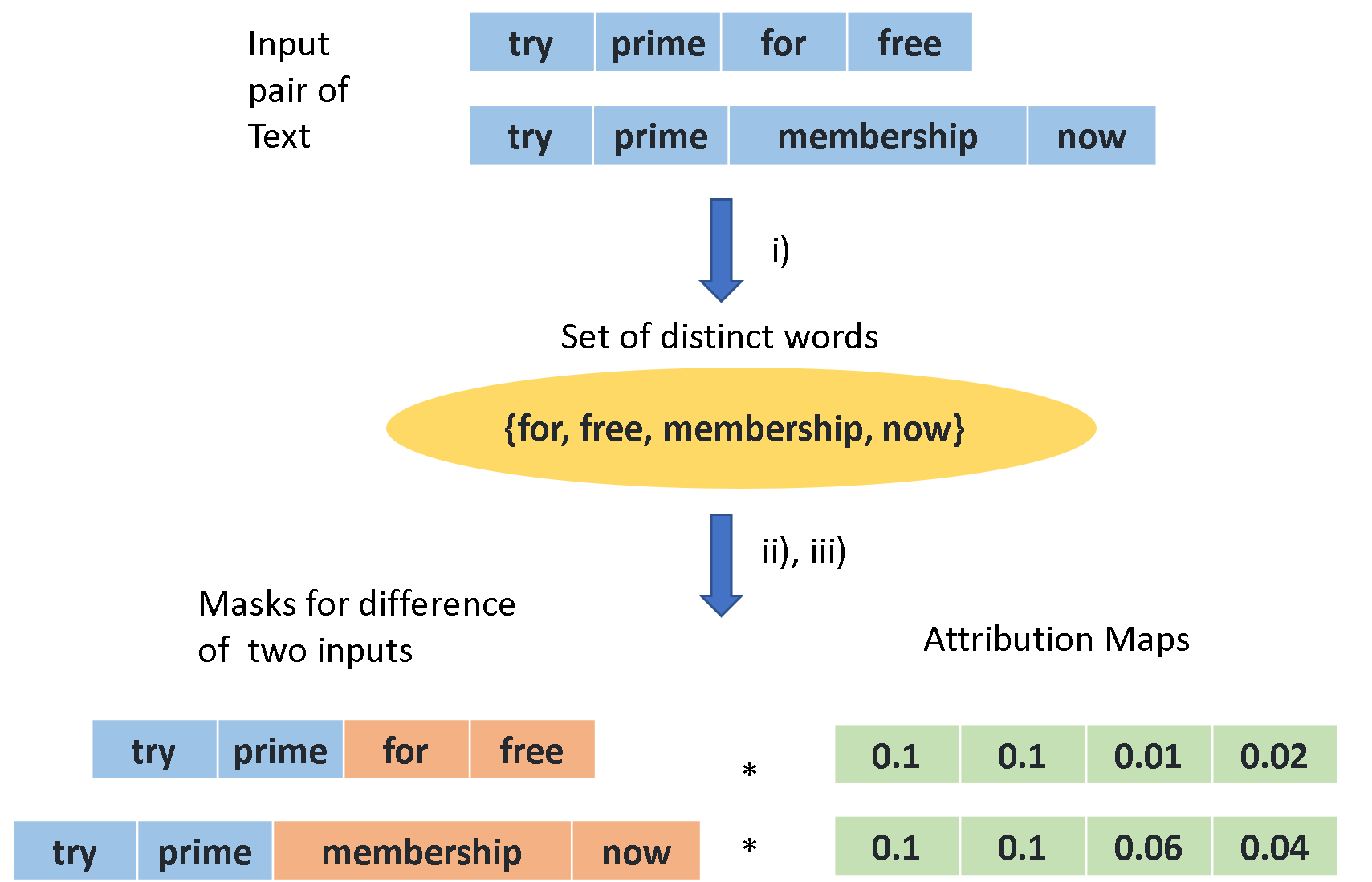}
    \caption{\small Visual explanation of text evaluation Algorithm~\ref{alg:eval_text}.}
    \label{fig:text_eval_exp}
\end{figure}

Specifically, we define the difference between two contents as the {\em difference set} $S=\{x|(x \in X \wedge x\notin X')\cup (x \notin X \wedge x\in X') \}$, the predicted attribution difference as $\Delta c(x) = \sum_{x \in (X' \cap S)}c(x) - \sum_{x \in (X \cap S)}c(x) $, and the actual success rate improvement as $\Delta Y = Y' - Y$. 
We postulate that a linear relationship exists between $(\Delta c(x), \Delta Y)$, which implies that marketers can improve the content's attractiveness by making modifications based on the insights.
Hence, we propose a method that first finds $\Delta c(x)$ by computing the difference set of instances $S$, and then evaluates the Pearson correlation coefficients $\rho$ across all possible control and treatment pairs within the same content domain in the dataset~\cite{benesty2009pearson}. The Pearson Correlation Coefficient used in our evaluation is defined as:
\begin{equation}\label{eq: corr}
        \rho = \frac{\text{cov}(\Delta c(x), \Delta Y)}{\sigma_{\Delta c(x)} \sigma_{\Delta Y}},
\end{equation}
where $\text{cov}(\Delta c(x), \Delta Y)$ is the covariance between $\Delta c(x)$ and $\Delta Y$, $\sigma_{\Delta c(x)}$ is the standard deviation of $\Delta c(x)$, and $\sigma_{\Delta Y}$ the standard deviation of $\Delta Y$. In our implementation, since we do not have direct access to $\Delta c(x)$ and $\Delta Y$, we compute the Pearson Correlation Coefficient $\rho$ of their surrogates. We denote the surrogates of $\Delta c(x)$ and $\Delta Y$ as $\mathbf{d}_C$ and $\mathbf{d}_Y$, respectively.
\begin{figure}[!t]
    \centering
    \includegraphics[width=0.39\textwidth]{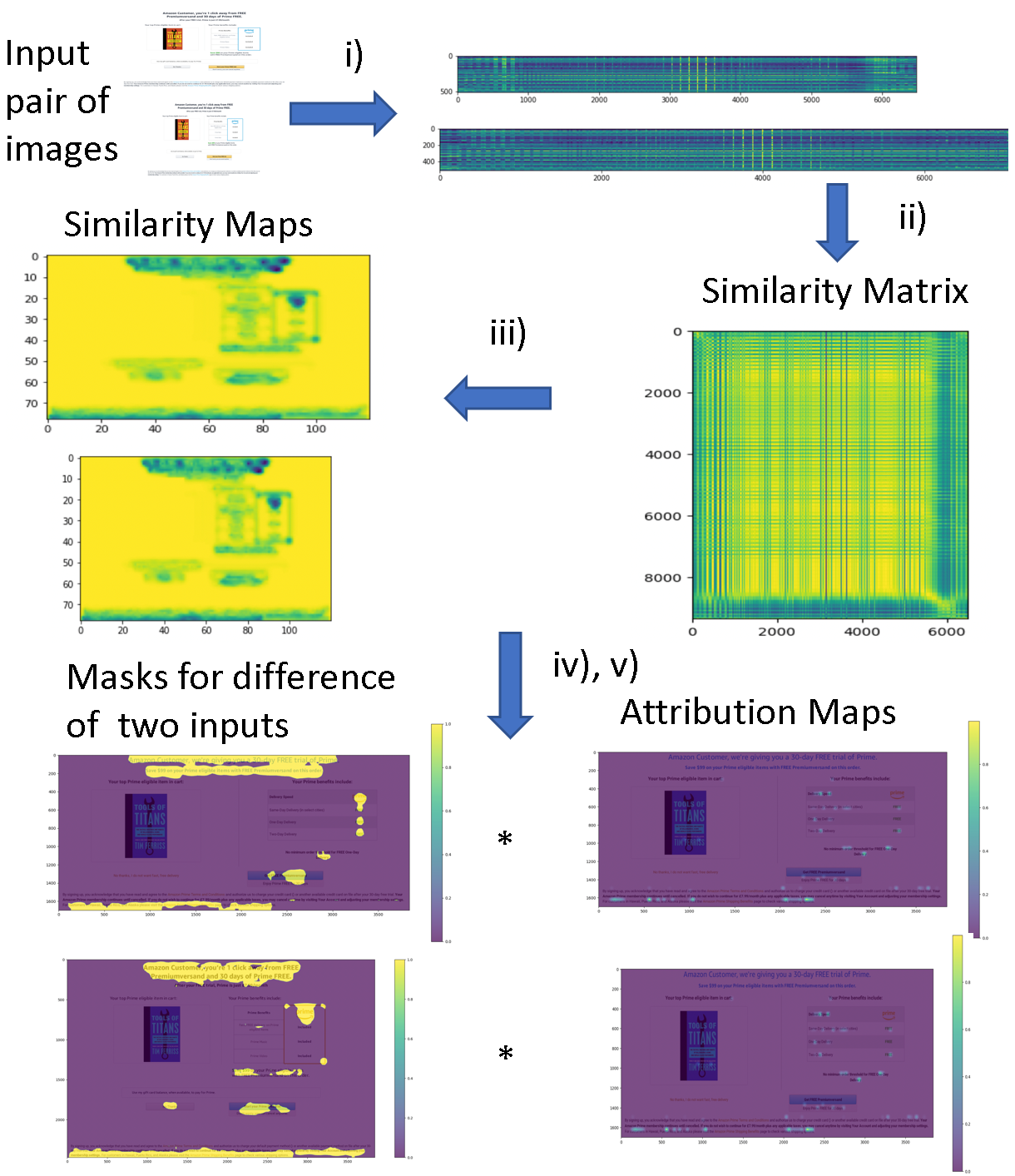}
    \caption{\small Visual explanation of image evaluation Algorithm~\ref{alg:eval_image}.}
    \label{fig:image_eval_exp}
\end{figure}

\paragraph{\textbf{Evaluation Algorithm Description.}}
We propose a generic three-step approach to evaluate insights in Algorithm~\ref{alg:eval_general}.
We also provide the pseudo-code of our implementation to evaluate insights for real-world data structures including images and text in the Appendix (Algorithms~\ref{alg:eval_text} and \ref{alg:eval_image}).
The general idea of Algorithm~\ref{alg:eval_general} is:
\begin{enumerate}
\item First, we find a difference set $S$ of two input samples, which represents the distinct elements that only appear in one of them.
Based on the difference set $S$, we generate two masks for two input samples.
Note that the masks retain the elements in the set $S$.
\item Then, we use the masks to obtain the inner product of the corresponding attribution maps for two samples, and compute the difference $d_C$ of the inner product results.
We call it the predicted summed attributions of modifications.
$d_C$ represents the total attributions when sample one is modified to sample two, or {\em vice versa}.
We also get $d_y$ by computing the difference in ground truth success rates of two samples.
\item Finally, we quantify the linear relationship between $d_C$ and $d_y$ by computing a correlation coefficient $\rho$ on the whole test dataset. 
\end{enumerate}

The resulting correlation coefficient represents how well, when the input sample is modified based on the attribution insight, can it contribute to the change of its ground truth success rate.
This metric is very useful in our digital marketing setting, where our goal is to provide deep insights generated by attributions to help digital marketers amend their content to improve its attractiveness.
To ensure the algorithms operates accurately, each pair of samples used to compute $d_c$ and $d_y$ must be a pair of control and treatment instances from the same content experiment. 

\paragraph{\textbf{Insight Examples.}}
Figures~\ref{fig:text_eval_exp} and~\ref{fig:image_eval_exp} provide visual explanations of our insights evaluation algorithm in text and image settings, respectively.
Figure~\ref{fig:text_eval_exp}, illustrating text evaluation, can be understood as follows.
In step $i$), we extract a set of words that only appear either in the control sentence or the treatment sentence.
In step $ii$), we use this set to create a mask for both sentences, where each element in the mask is $1$ (orange color in the figure) if the word in that position belongs to the set $S,$ or $0$ (blue color in the figure) if the word in that position does not appear in set $S$.
In step $iii$), we take the inner product of the masks with the attribution maps to produce $d_C$.

Figure~\ref{fig:image_eval_exp}, illustrating image evaluation, can be understood as follows.
Step $i$) creates the feature maps of the control and treatment images. After properly reshaping the feature maps, step $ii$) computes the similarity between every pair of control-treatment feature vectors, creating a similarity matrix.
Step $iii$) takes the matrix with maximum control-treatment similarity score for each location.
Step $iv$) thresholds the similarity maps to create masks for differences between control and treatment, and reshapes them to the same size as their corresponding attribution maps.
Step $v$) takes the inner product of the masks with the attribution maps to produce $d_C$.

\section{Experiments}\label{sec: results}
%
We evaluate our algorithm on the dataset described in Section~\ref{sec: data}. If a modality is missing, we use a zero vector to substitute the missing embedding.
We split the dataset into training, validation and test sets with a ratio of 50:10:40.
To evaluate the performance of our model on both existing and unseen content domains, we divide the test set into in-domain and out-of-domain subsets.
In-domain only contains content domains present in the training set, and out-of-domain includes market domains absent from it. 

\begin{figure*}[!htb]
\includegraphics[width=0.9\textwidth]{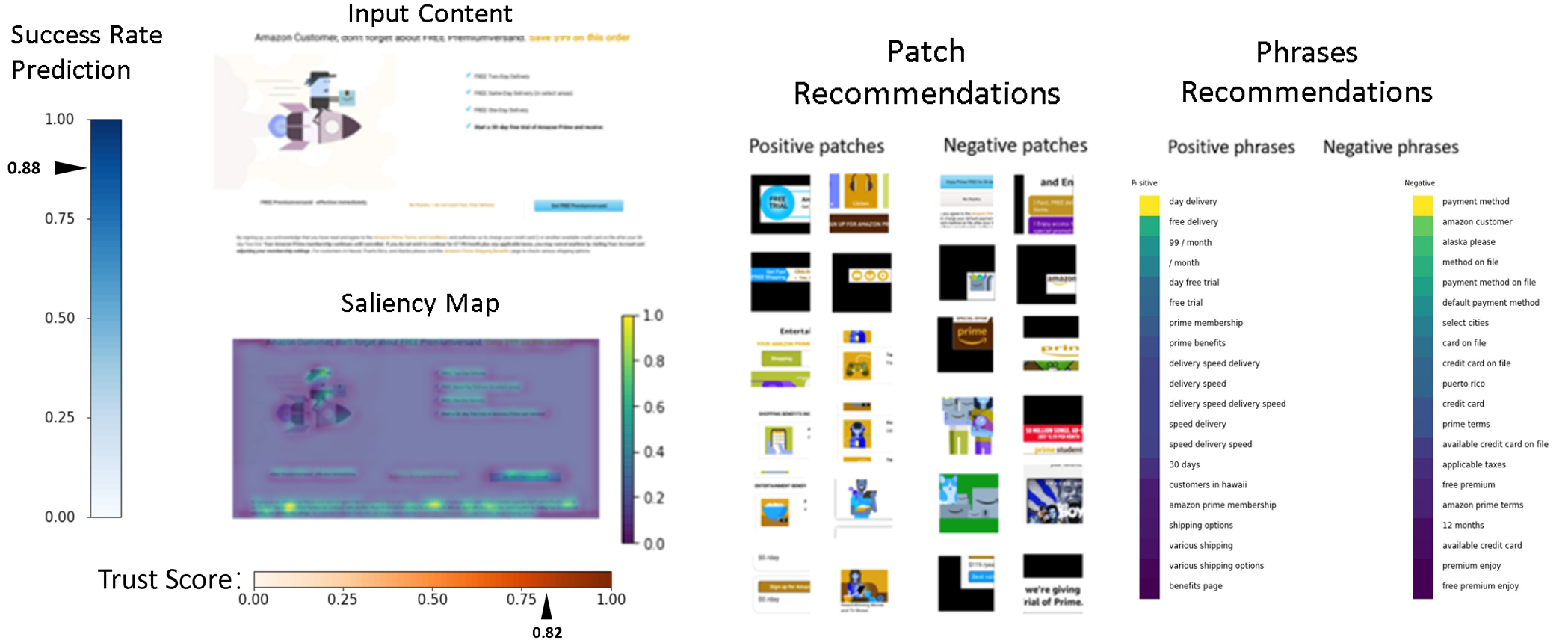}
\caption{Exemplar dashboard of our interactive system to refine existing content or design new content.
} 
\label{fig:dashboard}
\end{figure*}

\subsection{Model Specifications}
Here, we describe the training hyper-parameters used in our experiments. We use ResNet-18 as the image model.
The image model is trained via an Adam optimizer with a batch size of $32$, $\beta_1$ of $0.9$, $\beta_2$ of $0.999$ and a learning rate of $0.001$ for $50$ epochs.
During pre-training, we randomly crop a $512 \times 512$ patch from the image as input in order to limit GPU memory usage.
When we train the full multimodal model and infer new samples, we feed the whole image as the input.
Due to the high memory consumption of processing full-size screenshots (size ranging from $1000 \times 1000$ to $6000 \times 6000$), we freeze the weights of image models at this stage, avoiding GPU out-of-memory issues.
The text model uses BERT as its backbone, and is trained by an Adam optimizer with a batch size of $8$, $\beta_1$ of $0.9$, $\beta_2$ of $0.999$ and a learning rate of $0.001$ for $50$ epochs.

The domain, feature and regression modules are four-layer fully-connected MLP neural networks, with each layer followed by batch normalization and ELU activations.
ELU activation is often used in regression tasks~\cite{jesson2020identifying}.
The domain module and the feature module are trained via an Adam optimizer with a batch size of $512$, $\beta_1$ of $0.9$, $\beta_2$ of $0.999$ and a learning rate of $0.0005$ for $50$ epochs. 
The regression network is trained when we optimize the complete multimodal model.
After separately pretraining the image, text, domain and feature models, we train the whole multimodal model with a batch size of $32$, $\beta_1$ of $0.9$, $\beta_2$ of $0.999$ and a learning rate of $0.001$ for $50$ epochs.
See Appendix~\ref{sec: architectures} for neural network architecture details.

\subsection{Interactive Dashboard} \label{sec: dashboard}
For ease of use, we propose an interactive dashboard for digital marketers to visually work their content (see Figure~\ref{fig:dashboard}).
Our dashboard aims to provide similar functionality to~\cite{sinha2020designing}, but our framework turns out to be more powerful and comprehensive.
Specifically, our dashboard has merits that facilitate digital marketing design.
\begin{enumerate}
\item In~\cite{sinha2020designing}, the insights are restricted to the handcrafted features, which suffer from inefficient scalability and intuitiveness.
For example, it is unclear what to do with the insights on a specific attribute like ``lighting''.
Does it direct the marketer to increase the lighting of the whole page or a specific section?
In contrast, our insights are directly overlaid on the original content as a saliency map, as in Figure~\ref{fig:dashboard}.
\item Our system provides recommendations of design elements based on historical data. When the marketers design a website content in a specific content domain, our dashboard shows the patches and words/phrases with the highest average interpretation scores on historical data with the same content domain $D$.
\item Our system easily extends to new marketing content and novel features, thus our insights are not constrained to existing marketing content features.
\item Our success rate prediction is more accurate.
We compare several commonly used machine learning models with our proposed deep multi-modal method.
The results in Table~\ref{table: scoring} show our model outperforming the rest. We also present an insights ``Trust Score'', which is based on the insights evaluation results in Table~\ref{table: insights}.
\end{enumerate}

\begin{table*}[t]
\caption{Success rate prediction results for different models and modality combinations. We show the percentage decrease of RMSE and MAE for each model compared to GLM. } 
\centering 
\begin{tabular}{c|c| cc|cc} 
  & & \multicolumn{2}{c|}{In-domain test set} & \multicolumn{2}{c}{Out-of-domain test set} \\
  Model & Modality & RMSE change \FancyDownArrow & MAE change \FancyDownArrow  & RMSE change \FancyDownArrow & MAE change \FancyDownArrow \\
  \hline
  GLM  & Categorical Features & 0 \% & $0 \%$ & $0 \%$ & $0 \%$ \\
  MLP  & Categorical Features & \gr{-42\%} & \gr{-31\%} & \gr{-44\%} & \gr{-35\%}\\
  MLP  & Domain & \gr{-54\%}& \gr{-33\%}& \gr{-50\%} & \gr{-29\%}\\
  XGBoost & Categorical Features & \gr{-38\%} & \gr{-9\%} & \gr{-41\%} & \gr{-24\%}\\
  ResNet-18 & images & \gr{-25\%} & \color{red}{19\%} & \gr{-38\%} & \gr{-12\%} \\
  BERT & Text & \gr{-59\%} & \gr{-64\%} & \gr{-59\%} & \gr{-66\%} \\ 
  Multi-modal Neural Network & All modalities & \textbf{\bfgr{-68\%}} & \bfgr{-65\%} & \bfgr{-66\%} & \bfgr{-75\%}\\
\end{tabular}
\label{table: scoring} 
\end{table*}

\begin{table}[htb!]
\centering

\centering 
\caption{Results of insights evaluation. The performance metric is the percentage increase of Pearson Correlation Coefficient defined in Equation~\ref{eq: corr} for each attribution method compared to GradCam.} 
\begin{tabular}{cccccc} 

\multicolumn{1}{c|}{ } & \multirow{2}{*}{GradCam} & Integrated & Kernel & Feature &  \multirow{2}{*}{PCA}\\
\multicolumn{1}{c|}{ } & & Gradient & SHAP & Ablation & \\

\hline 
\multicolumn{1}{c|}{$\Delta \rho_{\text{text}}$ \FancyUpArrow}& 0\% & \gr{+288\%} & \gr{+109\%} & \color{red}{-14\%}& \bfgr{+493\%} \\
\multicolumn{1}{c|}{$\Delta \rho_{\text{image}}$ \FancyUpArrow}& 0\% & \gr{+55\%} & \color{red}{-18\%} & \gr{+0.5\%}& \bfgr{+145\%}
\end{tabular}
\label{table: insights}
\end{table}

\subsection{Evaluation}
\label{sec:evaluation}
We evaluate our method using multiple methods.
We quantitatively evaluate the success rate prediction, comparing our proposed multi-modal neural network to competing methods. 
We then report the predicted causal effect of applying our insights to improve content using our proposed correlation metric.
Qualitatively, we exhibit some feedback of using our interactive dashboard in Section~\ref{sec: dashboard} to design marketing contents from real-world digital marketers.

\begin{algorithm}[!htb]
\caption{Pairwise Accuracy.}\label{alg:pair_acc}
\KwData{Predictions $\mathbf{\widehat{y}}=[\widehat{y}_1,\ldots,\widehat{y}_n  ]$, truth $\mathbf{y}=[y_1,\ldots,y_n]$} 
\KwResult{Pairwise accuracy score $s$.} 
Initialize $\text{count}=0$ and $\text{hit}= 0$ .\\
\textbf{for every distinct pair $(\widehat{y}_i, \widehat{y}_j)$ and $(y_i, y_j)$ in dataset} {do}\;
\qquad \textbf{if} $\text{sign}(\widehat{y}_i-\widehat{y}_j) = \text{sign}(y_i-y_j)$ \;
\qquad \qquad $\text{hit} = \text{hit} + 1$\;
\qquad \textbf{end} \;
\qquad $\text{count} = \text{count} + 1$ \;
 \textbf{end} \;
  $s \leftarrow \text{hit}/\text{count}$\;
\textbf{Output} $s$   \\
\end{algorithm}

\paragraph{\textbf{Success rate and pairwise prediction}}
Table~\ref{table: scoring} shows the success rate prediction results of different scoring models on our dataset. Here, we report the change in Root Mean Square Error (RMSE) and Mean Absolute Error (MAE), both commonly used to evaluate the performance of regression models.
We test the Generalized Linear Model used in~\cite{sinha2020designing}; MLP and XGBoost using only categorical features extracted from text and images, which are typically used in industrial applications; and deep learning models that take a single modality as input (BERT with text as input and ResNet-18 with image as input). The results show that our multi-modal neural network outperforms all competing methods.

During content experimentation, marketers often target a content to iterate on and improve.
Then they conduct an experiment to compare the control content with its modified counterpart(s) ({\em i.e.}\ treatments).
    We use the pairwise ranking accuracy~\cite{ackerman2011evaluating} between the control and each treatment counterpart to evaluate the performance of our models. Algorithm~\ref{alg:pair_acc} details how pairwise accuracy is computed.
    The \textit{Pairwise Accuracy} of our proposed model achieves a relative percentage increase of \gr{+38\%} on an out-of-domain test set when compared to GLM.
This result shows that our neural network model is much more accurate for marketers in real-world use-cases.

\paragraph{\textbf{Evaluating Insights}}
In Table~\ref{table: insights}, we evaluate the insights generated from the trained deep neural networks using our proposed evaluation scheme (see Section~\ref{sec: insights_eval}).
In our dataset, we have multiple treatments related to a given control, requiring $O(n^2)$ time to compute $d_C$ and $d_y$.
We avoid such computational complexity by only comparing control with the best performing treatment in the same content domain.

For text data, Integrated Gradients performs the best among GradCam, Kernel SHAP and Feature Ablation. After we integrate these interpretation methods together by PCA, our method yields the highest correlation score. PCA returns a relative percentage increase of  \gr{+493\%}, which is a very high correlation score. The result of PCA indicates a strong correlation between insights and success rate improvement, suggesting that the insights are trustworthy.
Marketers should consider modifying their templates based on the insight attribution scores, and the insights-guided modification are highly likely to improve the success rate.

For image data, all above-mentioned attribution methods are too slow or intractable, as the size of image inputs is much larger than text inputs, taking too much time to compute attributions for all input pixels. To run the experiment in a reasonable time, we discard the very large images that has more than $5e+06$ pixels and evaluate the insights of the remaining image data. From the results, we still see the pattern that Integrated Gradients and PCA methods outperform GradCam, with Integrated Gradient and PCA posting a correlation increase of \gr{+55\%} and \gr{+145\%} respectively. We hypothesize that Integrated Gradient is more accurate since it computes attributions on the original image, as opposed to computing it on the intermediate activations, as with GradCam. PCA integrates different aspects of attributions and captures the shared variance of attribution maps from GradCam, Integrated Gradients, KernelShap and Feature Ablations, leading to the best results.

\paragraph{\textbf{User Experience}}

To further demonstrate the claims in our paper, we launched a demo of the functionality discussed in Section~\ref{sec: dashboard}. The demo dashboard looks similar to Figure~\ref{fig:dashboard}, including a saliency map that highlights which parts of the input content to keep or redesign, and recommended phrase and patch insights to act on. 
The demo has been shown to tens of professional digital marketers, with mostly positive feedback. 

Here is a positive feedback example, which highlights the usefulness of our framework in facilitating marketing content design:
``The new demo visualization insights helped make analyzing our current templates faster - allowing marketers to spend more time identifying opportunities, create hypotheses, and test new experiences based on the results. In addition, the positive and negative contribution saliency maps enable marketers to select what areas of a template may have the highest impact during experimentation. We are looking forward to continue working to develop this tool and use it to help with successful experiments!''

In the above user's feedback, the marketer praises our positive and negative contribution saliency maps.
In our implementation, the positive (negative) contribution map is based on the absolute value of the positive (negative) part of the attribution map.
This visualization makes it easy for users to identify the positive and negative impact of the input content. 

\section{Related Works}
In this section, we briefly discuss the existing works related to our topic, including modeling digital marketing contents, related deep learning approaches for text and image recommendations, and evaluation metrics for attribution methods. Note that none of these related works fully scales and solves our problem, especially as we define distinct tasks in Sections~\ref{sec: insights} and~\ref{sec: insights_eval}. 

\textbf{Modeling Digital Marketing Contents.} The problem of modeling digital marketing content has triggered substantial research efforts \cite{fong2019image, sinha2020designing, wang2022data, zhou2020product} over the past decade. \citet{fong2019image} developed a machine learning pipeline to classify advertising images based on their quality. \citet{wang2022data} combines deep neural network and evolutionary algorithm to predict optimal personalized marketing strategy for better incomes. \citet{zhou2020product} proposes a recommendation algorithm based on recurrent neural network and distributed expression for recommending new products to consumers based on their browsing history. The above-mentioned works are out of our scope, as we focus on extracting insights from deep models to help digital marketers improve their content.
The closest research to ours is \citet{sinha2020designing}, which aims to improve the attractiveness of contents by providing AI insights. Nevertheless, they use a much simpler machine learning pipeline than ours, such that our framework has better prediction accuracy and more interpretable insights. Besides, they don't propose an insights evaluation metric, making us the first researchers to quantitatively examine the effectiveness of generated marketing AI insights.

\textbf{Related Deep Learning Approaches.}
Among the reproducible deep learning approaches, our recommendation is quite similar to prototype learning.
Prototype learning is
a form of case-based reasoning \cite{kolodner1992introduction, schmidt2001cased}, which draws conclusions
for new inputs by comparing them with a few exemplar cases (i.e prototypes) in the problem domain \cite{chen2019looks, li2018deep}. It is a natural practice in
our day-to-day problem-solving process. For example, physicians
perform diagnosis and make prescriptions based on their experience with past patients~\cite{dutra2011upgrades, geng2018temporal}, and mechanics predict potential malfunctions by recalling vehicles exhibiting similar symptoms~\cite{Geng2023SAmQ}. Prototype
learning imitates human problem-solving processes for better
interpretability. 
Recently the concept has been incorporated in convolutional neural networks to build interpretable image classifiers
\cite{chen2019looks, li2018deep}. 
Our framework is somehow similar to ProtoPNet~\cite{chen2019looks}, in the sense that we both first highlight the salient areas and then make recommendations. ProtoPNet outputs the recommendations that explain the image classification results, while our recommendations focus on improving the attractiveness scores of the current input. 
So far, prototype learning is not yet explored for modeling and improving digital marketing contents.  Our method can be seen as learning prototypes that increase the regression scores, a new problem that we leave for future work. 

\textbf{Evaluating Attribution Methods.}
Recent research have proposed several metrics to evaluate attribution methods, which can be divided into two categories: Sanity Checks and Localization-Based Metrics. Sanity Checks \cite{adebayo2018sanity, rao2022towards, agarwal2022openxai} are designed to examine the basic properties of attribution methods according to faithfulness, stability and
fairness. 
 We aim at quantifying the effectiveness of attribution methods in real-world applications though. Hence our evaluation scheme examines the relationship between insights-guided modifications and the ensuing change in the actual success rate. 
 Localization-Based Metrics measure how well attributions coincide with object
bounding boxes or image grid cells that contains the key objects explaining the classification results~\cite{bohle2021convolutional, cao2015look, fong2017interpretable}. In our scenario, we do not have the ground-truth bounding boxes, and our attribution methods explain the regression model. Thus localization-based metrics do not apply.

\section{Conclusion}
This paper constitutes the first attempt to use deep learning to facilitate the digital marketing design process.
Our multimodal neural network outperforms competing methods in predicting success rates, and leverages neural attribution methods to provide insights that guide digital marketers to improve their existing design. 
Our approach is modular and generalizable, and individual neural components can be easily replaced as the state-of-the-art evolves. 
This work underscores the need to explore causal-aware models for modeling content experimentation, which we leave as future work. Additionally, our system's output insights can be further improved by high-capacity language and vision models such as ChatGPT \cite{OpenAI2023GPT4TR} and SAM \cite{kirillov2023segment}. These models can provide clearer and more actionable instructions for human experts. Besides, our proposed insights evaluation methods may have broader impact on other real-world use-cases such as in healthcare, finance, bank sales etc. For example, quantifying the estimated contributions of biological risk factors on healthcare costs~\cite{lee2022quantifying} or examining the effectiveness of a predicted business decision from an AI agent on the company's income/loss~\cite{alaluf2022reinforcement}.


\section*{ACKNOWLEDGMENTS}
The authors would like to thank the anonymous reviewers for their insightful comments.
This research was supported by ONR N00014-18-1-2871-P00002-3.
Special thanks to Amazon AWS for generously providing the computing resources necessary for this work.


\clearpage
\newpage
\bibliographystyle{ACM-Reference-Format}
\balance
\bibliography{ref}

\newpage
\appendix
\section*{Appendix}

\section{Detailed Insight Evaluation Algorithms}

\begin{algorithm}[!htb]
\caption{Evaluate attribution results of a text model.}\label{alg:eval_text}
\KwData{Input dataset $\{(x_1, y_1), (x_2, y_2), ..., (x_N, y_N)\}$ where $x_i \in \mathbb{R}^n$ is the input and $y_i \in \mathbb{R}$ is the label, 
and their corresponding attribution maps $\{C_1, C_2, ..., C_n\}$ where $C_i\in \mathbb{R}^n$.}
\KwResult{Correlation coefficient $\rho$.}
  Initialize $k\leftarrow 0$, $\mathbf{d_C} \leftarrow \vec{0}^{\frac{n*(n-1)}{2}}$ and  $\mathbf{d_y}\leftarrow \vec{0}^{\frac{n*(n-1)}{2}}$. \\
 1) Find the difference: \\
\textbf{For every pair of control and treatment $\{x_i,x_j\}$ in dataset} \textbf{do}:     \\
      \par  \textbf{i.} Compute the distinct elements set $S_{i,j}$,  such that the attributes in $S_{i,j}$ can be only found in $x_i$ or $x_j$. \\
        \par \qquad   $S_i \leftarrow \text{set}({x_i})$, \, $S_j \leftarrow \text{set}({x_j})$ ;        
 \par \qquad   $S_{i, j} \leftarrow S_i  \cup S_j - S_i \cap S_j$ ; 

      \par \textbf{ii.} Compute $P_i$ and $P_j$, indicator vectors where $P_i := \{p_i^s\}_{s=1,2,...n}$ such that $p_i^s = 1$ if $x_i^s \in S_{i,j}$ and $p_i^s = 0$ if $x_i^s \notin S_{i,j}$, and $P_j := \{p_j^s\}_{s=1,2,...n}$ such that $p_j^s = 1$ if $x_j^s \in S_{i,j}$ and $p_j^s = 0$ f $x_j^s \notin S_{i,j}$.

      \par \textbf{iii.} Compute $d_C = \text{sign}(y_i - y_j)(P_i^TC_i - P_j^TCj)$ as the sum of predicted attributions difference, and $d_y = |y_i - y_j|$ as the actual success rate improvements. \\
      Update:
        \par  \qquad $\mathbf{d_C}[k] \leftarrow d_C$, \, $\mathbf{d_y}[k] \leftarrow d_y$ \;
         \par  \qquad $k \leftarrow k + 1$ \;
      \textbf{end} \;
 2) Compute Pearson Correlation $\rho$ between $\mathbf{d_C}$ and $\mathbf{d_Y}$ . \\
\textbf{Output $\rho$}.
\end{algorithm}
\begin{algorithm}[!htb]
\caption{Evaluate attribution results of an image model.}\label{alg:eval_image}
\KwData{Input dataset $\{(x_1, y_1), (x_2, y_2), ..., (x_N, y_N)\}$ where $x_i \in \mathbb{R}^{m\times n \times Z}$ is an RGB input image and $y_i \in \mathbb{R}$ is the label of the image, their corresponding attribution maps $\{C_1, C_2, ..., C_n\}$ where $C_i\in \mathbb{R}^{m\times n}$, and the vision model $\Phi(\cdot)$ that can extract features of input images.} 
\KwResult{Correlation coefficient $\rho$.}
  Initialize $k\leftarrow 0$, $\mathbf{d_C} \leftarrow \vec{0}^{\frac{n*(n-1)}{2}}$ and  $\mathbf{d_y}\leftarrow \vec{0}^{\frac{n*(n-1)}{2}}$. \\
 1) Find the difference: \\
\textbf{For every pair of control and treatment $\{x_i,x_j\}$ in dataset} \textbf{do}:     \\
      \textbf{i.} Compute the feature maps of inputs: \\
      \qquad $\mathbf{A}_i \leftarrow \Phi(x_i)$,  $\mathbf{A}_j\leftarrow \Phi(x_j)$ \;
    reshape $\mathbf{A}_i$ and $\mathbf{A}_j$: \\ 
     \qquad $\mathbf{A}_i  \leftarrow \text{reshape}(\mathbf{A}_i, (m'  n',  Z')) $,\\
     \qquad $\mathbf{A}_j  \leftarrow \text{reshape}(\mathbf{A}_i, ( Z', m' n')) $\;
          \textbf{ii.} Compute the Cosine Similarity matrix between every feature vector in $\mathbf{A}_i $ and every feature vector in $\mathbf{A}_j $: \\
  \qquad    $\mathbf{S}_{i,j}[k,l] \leftarrow \mathlarger{\frac{<\mathbf{A}_i[k, :], \mathbf{A}_j[:, l]>}{|\mathbf{A}_i[k, :]||\mathbf{A}_j[:, l]|}},$ \\
  \qquad $\forall k=0,1,2,..., m'n' ,\quad \forall l=0,1,2...,m'n' $\;
      \textbf{iii.} Take the maximum similarity scores for each location in $x_i$ and $x_j$: \\
      \qquad  $\mathbf{d}_{x_i}[i] \leftarrow \max\limits_{l}\mathbf{S}[i, l],\quad \forall i=0,1,2...,m'n' $ \;
         \qquad $\mathbf{d}_{x_j}[j] \leftarrow \max\limits_{k}\mathbf{S}[k, j], \quad \forall j=0,1,2...,m'n'$
         
          \textbf{iv.} Threshold $\mathbf{d}_{x_i}$, $\mathbf{d}_{x_j}$ and resize them to the same dimension as $C_i$, $C_j$: \\
          \qquad $P_i \leftarrow \text{threshold}(\mathbf{d}_{x_i})$, \, $P_j \leftarrow \text{threshold}(\mathbf{d}_{x_j})$ \;
        Resize $P_i$ and $P_j$: \\
                  \qquad $P_i \leftarrow \text{resize}(P_i, (m, n))$, \, $P_j \leftarrow \text{resize}(P_j, (m, n))$ \;
              \textbf{v.} Compute predicted attribution difference $d_C$ and acutal success rate improvement $d_y$: \\
     \qquad  $d_C = \text{sign}(y_i - y_j)(\sum P_i \odot C_i - \sum P_j  \odot Cj)$ \;
  \qquad  $d_y = |y_i - y_j|$ \;
       
        Update:
        \par \qquad \qquad $\mathbf{d_C}[k] \leftarrow d_C$, \, $\mathbf{d_y}[k] \leftarrow d_y$ \;
         \par \qquad \qquad $k \leftarrow k + 1$ \;
         
      \textbf{end} \;
 2) Compute Pearson Correlation $\rho$ between $\mathbf{d_C}$ and $\mathbf{d_Y}$ . \\

\textbf{Output $\rho$}.
\end{algorithm}

\section{Neural Network Architecture Details}
\label{sec: architectures}
In Table~\ref{table: architectures}, the convolutional layer is denoted as "Conv", followed by the kernel size, stride, padding and number of filters. "fc" means fully-connected layer and the output hidden units is provided after the dash. "ELU", "ReLU" and "Sigmoid" represent the non-linear functions. "GlobalAveragePooling2D"  is the global average pooling operation in the spatial dimension of the tensors, functioning the same as Keras' Global Average Pooling 2D~\cite{2Davgpooling}. "ResBlock" is the standard ResNet block~\cite{he2016deep}. In the brackets, we provide the kernel size, stride, and number of filters.  "TransformerLayer" is the standard layer in a transformer~\cite{vaswani2017attention}. In the brackets, we provide the size of hidden layers and the number of attention heads.

\begin{table}[!htb]
\caption{The architecture of each component in our multimodal neural network.} 
\label{table: architectures}
\centering 
\begin{tabular}{|c| c|} 
 \multicolumn{2}{c}{$\text{ResNet}(\cdot)$}  \\
\hline

Layer & Type \\
\hline
1 & Conv(3, 1, 1)-32 + ReLU()  \\
2 & ResBlock(3, 1, 32)  \\
3 & ResBlock(3, 2, 32)  \\
4 & ResBlock(3,2, 32)  \\
5 & ResBlock(3,2, 32)  \\
6 & BatchNorm()+ReLU() \\
7 & GlobalAveragePooling2D()  \\
\hline

\end{tabular}
\\
\begin{tabular}{|c| c|} 
\multicolumn{2}{c}{ }  \\
 \multicolumn{2}{c}{$\text{BERT}(\cdot)$}  \\
\hline
Layer & Type \\
\hline
1-12 & TransformerLayers(768, 12)  \\
\hline

\end{tabular}
\\

\begin{tabular}{|c| c|} 
\multicolumn{2}{c}{ }  \\
 \multicolumn{2}{c}{$\text{MLP}_1(\cdot)$}  \\
\hline
Layer & Type \\
\hline
1 & fc-512 + BatchNorm + ELU()  \\
2 & fc-1024 + BatchNorm + ELU()  \\
3 & fc-1024 + BatchNorm + ELU()  \\
4 & fc-512 + BatchNorm + ELU()  \\
\hline

\end{tabular}
\\
\begin{tabular}{|c| c|} 
\multicolumn{2}{c}{ }  \\
 \multicolumn{2}{c}{$\text{MLP}_2(\cdot)$}  \\
\hline
Layer & Type \\
\hline
1 & fc-512 + BatchNorm + ELU()  \\
2 & fc-1024 + BatchNorm + ELU()  \\
3 & fc-1024 + BatchNorm + ELU()  \\
4 & fc-512 + BatchNorm + ELU()  \\
\hline

\end{tabular}
\\
\begin{tabular}{|c| c|} 
\multicolumn{2}{c}{ }  \\
 \multicolumn{2}{c}{$\text{MLP}_3(\cdot)$}  \\
\hline
Layer & Type \\
\hline
1 & fc-512 + BatchNorm + ELU()  \\
2 & fc-1024 + BatchNorm + ELU()  \\
3 & fc-1024 + BatchNorm + ELU()  \\
4 & fc-1  + Sigmoid() \\
\hline

\end{tabular}

\end{table}

\section{Additional Results}
In this section, we offer an additional result to support the finding in our main paper. Specifically, Figure~\ref{fig: rmse_mae_per_domain} compares RMSE and MAE of our multimodal neural network and the Generalized Linear Model (GLM) on each domain. Notably, each bar for RMSE and MAE only computed on multimodal data from each respective domain. The objective here is to underscore the consistent error reduction achieved by our multimodal neural network across a variety domains.

\begin{figure}[!b]
    \centering
    \includegraphics[width=0.45\textwidth]{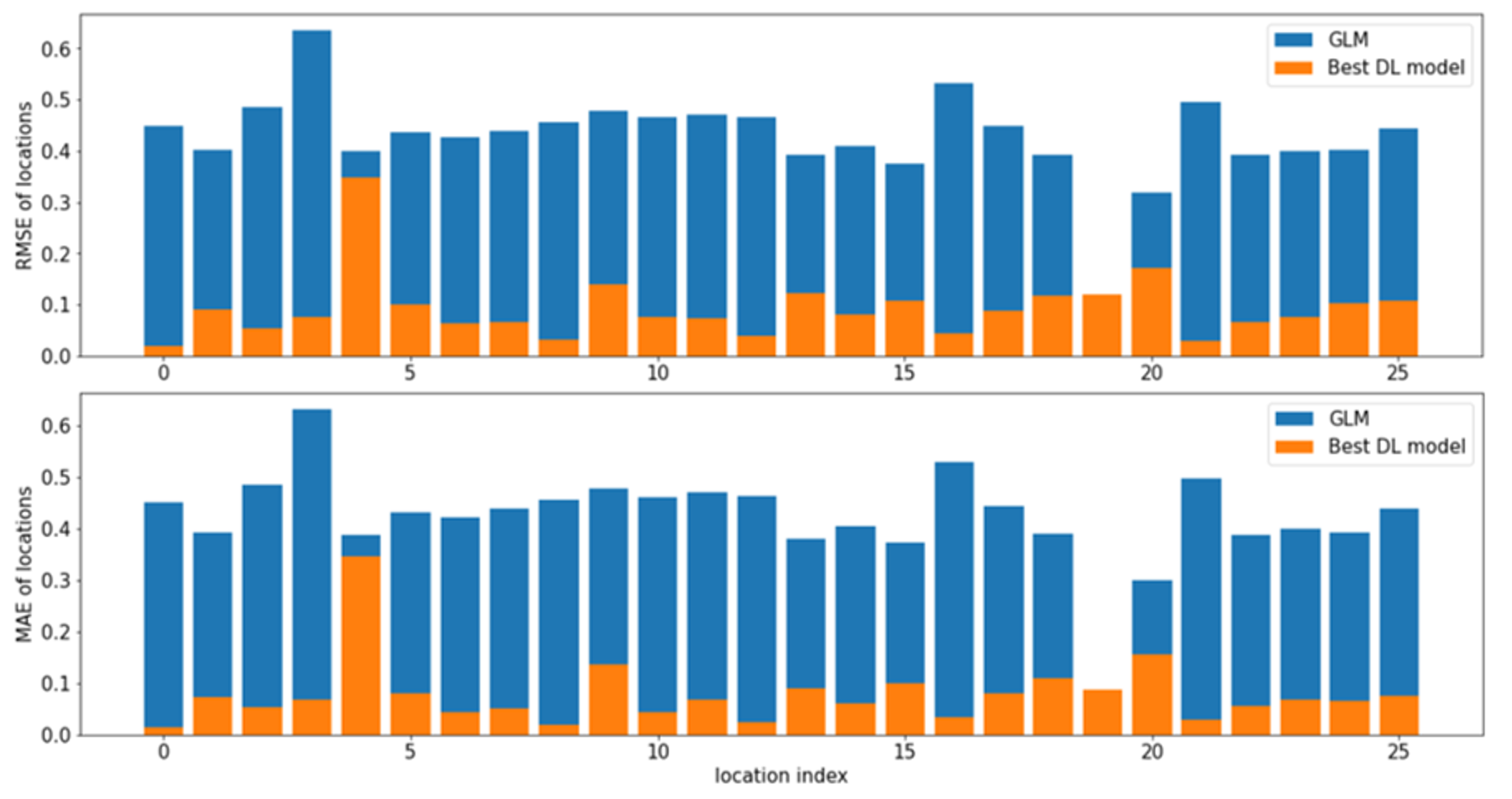}
    \caption{RMSE(top) and MAE(bottom) of GLM(blue) and our multimodal neural network(orange) evaluated on each domain.} 
    \label{fig: rmse_mae_per_domain}
\end{figure}

\end{document}